%% file: conv_reg.tex
\documentclass{article}

\input{math_commands}

\usepackage{microtype}
\usepackage{graphicx}
\usepackage{subfigure}
\usepackage{booktabs} 
\usepackage{amssymb}
\usepackage{amsmath}
\usepackage{nicefrac}
\usepackage{xcolor}

\newtheorem{lemma}{Lemma}
\newtheorem{corollary}{Corollary}

\newtheorem{prop}{Proposition}
\newtheorem{proof}{Proof}

\newtheorem{innercustomgeneric}{\customgenericname}
\providecommand{\customgenericname}{}
\newcommand{\newcustomtheorem}[2]{%
  \newenvironment{#1}[1]
  {%
   \renewcommand\customgenericname{#2}%
   \renewcommand\theinnercustomgeneric{##1}%
   \innercustomgeneric
  }
  {\endinnercustomgeneric}
}
\newcustomtheorem{manualtheorem}{Theorem}

\usepackage[accepted]{icml2021}

\icmltitlerunning{Convex Regularization in Monte-Carlo Tree Search}

\begin{document}

\twocolumn[
\icmltitle{Convex Regularization in Monte-Carlo Tree Search}



\icmlsetsymbol{equal}{*}

\begin{icmlauthorlist}
\icmlauthor{Tuan Dam}{1}
\icmlauthor{Carlo D'Eramo}{1}
\icmlauthor{Jan Peters}{1,2}
\icmlauthor{Joni Pajarinen}{1,3}
\end{icmlauthorlist}

\icmlaffiliation{1}{Department of Computer Science, Technische Universit{\"a}t Darmstadt, Germany}
\icmlaffiliation{2}{Robot Learning Group, Max Planck Institute for Intelligent Systems,T{\"u}bingen, Germany}
\icmlaffiliation{3}{Computing Sciences, Alto University, Finland}

\icmlcorrespondingauthor{Tuan Dam}{tuan.dam@tu-darmstadt.de}
\icmlcorrespondingauthor{Carlo D'Eramo}{carlo.deramo@tu-darmstadt.de}
\icmlcorrespondingauthor{Jan Peters}{mail@jan-peters.net}
\icmlcorrespondingauthor{Joni Pajarinen}{pajarinen@ias.tu-darmstadt.de}
\icmlkeywords{Monte-Carlo tree search, Reinforcement learning, Convex regularization}

\vskip 0.3in
]



\printAffiliationsAndNotice{}  

\begin{abstract}
Monte-Carlo planning and Reinforcement Learning~(RL) are essential to sequential decision making. The recent AlphaGo and AlphaZero algorithms have shown how to successfully combine these two paradigms to solve large scale sequential decision problems. These methodologies exploit a variant of the well-known UCT algorithm to trade off the exploitation of good actions and the exploration of unvisited states, but their empirical success comes at the cost of poor sample-efficiency and high computation time. In this paper, we overcome these limitations by introducing the use of convex regularization in Monte-Carlo Tree Search~(MCTS) to drive exploration efficiently and to improve policy updates. First, we introduce a unifying theory on the use of generic convex regularizers in MCTS, deriving the first regret analysis of regularized MCTS and showing that it guarantees an exponential convergence rate. Second, we exploit our theoretical framework to introduce novel regularized backup operators for MCTS, based on the relative entropy of the policy update and, more importantly, on the Tsallis entropy of the policy, for which we prove superior theoretical guarantees. We empirically verify the consequence of our theoretical results on a toy problem. Finally, we show how our framework can easily be incorporated in AlphaGo and we empirically show the superiority of convex regularization, w.r.t. representative baselines, on well-known RL problems across several Atari games.
\end{abstract}

\section{Introduction}
Monte-Carlo Tree Search~(MCTS) is a well-known algorithm to solve decision-making problems through the combination of Monte-Carlo planning and an incremental tree structure~\citep{coulom2006efficient}. MCTS provides a principled approach for trading off between exploration and exploitation in sequential decision making. Moreover, recent advances have shown how to enable MCTS in continuous and large problems~\citep{silver2016mastering,yee2016monte}. Most remarkably, AlphaGo~\citep{silver2016mastering} and AlphaZero~\citep{silver2017bmastering,silver2017amastering} couple MCTS with neural networks trained using Reinforcement Learning (RL)~\citep{sutton1998introduction} methods, e.g., Deep Q-Learning~\citep{mnih2015human}, to speed up learning of large scale problems. In particular, a neural network is used to compute value function estimates of states as a replacement of time-consuming Monte-Carlo rollouts, and another neural network is used to estimate policies as a probability prior for the therein introduced PUCT action selection strategy, a variant of well-known UCT sampling strategy commonly used in MCTS for exploration~\citep{kocsis2006improved}. Despite AlphaGo and AlphaZero achieving state-of-the-art performance in games with high branching factor like Go~\citep{silver2016mastering} and Chess~\citep{silver2017bmastering}, both methods suffer from poor sample-efficiency, mostly due to the polynomial convergence rate of PUCT~\citep{xiao2019maximum}. This problem, combined with the high computational time to evaluate the deep neural networks, significantly hinder the applicability of both methodologies.

In this paper, we provide a theory of the use of convex regularization in MCTS, which proved to be an efficient solution for driving exploration and stabilizing learning in RL~\citep{schulman2015trust,schulman2017equivalence, haarnoja2018soft, buesing2020approximate}. In particular, we show how a regularized objective function in MCTS can be seen as an instance of the Legendre-Fenchel transform, similar to previous findings on the use of duality in RL~\citep{mensch2018differentiable,geist2019theory,nachum2020duality} and game theory~\citep{shalev2006convex,pavel2007duality}. Establishing our theoretical framework, we can derive the first regret analysis of regularized MCTS, and prove that a generic convex regularizer guarantees an exponential convergence rate to the solution of the regularized objective function, which improves on the polynomial rate of PUCT. These results provide a theoretical ground for the use of arbitrary entropy-based regularizers in MCTS until now limited to maximum entropy~\citep{xiao2019maximum}, among which we specifically study the relative entropy of policy updates, drawing on similarities with trust-region and proximal methods in RL~\citep{schulman2015trust,schulman2017proximal}, and the Tsallis entropy, used for enforcing the learning of sparse policies~\citep{lee2018sparse}. Moreover, we provide an empirical analysis of the toy problem introduced in~\citet{xiao2019maximum} to evince the practical consequences of our theoretical results for each regularizer. Finally, we empirically evaluate the proposed operators in AlphaGo, on several Atari games, confirming the benefit of convex regularization in MCTS, and in particular the superiority of Tsallis entropy w.r.t. other regularizers.

\section{Preliminaries}\label{S:background}
\subsection{Markov Decision Processes}
We consider the classical definition of a finite-horizon Markov Decision Process~(MDP) as a $5$-tuple $\mathcal{M} = \langle \mathcal{S}, \mathcal{A}, \mathcal{R}, \mathcal{P}, \gamma \rangle$, where $\mathcal{S}$ is the state space, $\mathcal{A}$ is the finite discrete action space, $\mathcal{R}: \mathcal{S} \times \mathcal{A} \times \mathcal{S} \to \mathbb{R}$ is the reward function, $\mathcal{P}: \mathcal{S} \times \mathcal{A} \to \mathcal{S}$ is the transition kernel, and $\gamma \in [0, 1)$ is the discount factor. A policy $\pi \in \Pi: \mathcal{S} \times \mathcal{A} \to \mathbb{R}$ is a probability distribution of the event of executing an action $a$ in a state $s$. A policy $\pi$ induces a value function corresponding to the expected cumulative discounted reward collected by the agent when executing action $a$ in state $s$, and following the policy $\pi$ thereafter: $Q^\pi(s,a) \triangleq \mathbb{E} \left[\sum_{k=0}^\infty \gamma^k r_{i+k+1} | s_i = s, a_i = a, \pi \right]$, where $r_{i+1}$ is the reward obtained after the $i$-th transition. An MDP is solved finding the optimal policy $\pi^*$, which is the policy that maximizes the expected cumulative discounted reward. The optimal policy corresponds to the one satisfying the optimal Bellman equation~\citep{bellman1954theory} $Q^*(s,a) \triangleq \int_{\mathcal{S}} \mathcal{P}(s'|s,a)\left[ \mathcal{R}(s,a,s') + \gamma \max_{a'}Q^*(s',a') \right] ds'$, and is the fixed point of the optimal Bellman operator $\mathcal{T}^*Q(s,a) \triangleq \int_{\mathcal{S}}\mathcal{P}(s'|s,a)\left[\mathcal{R}(s,a,s') + \gamma \max_{a'}Q(s',a') \right] ds'$. Additionally, we define the Bellman operator under the policy $\pi$ as $\mathcal{T}_\pi Q(s,a) \triangleq \int_{\mathcal{S}}\mathcal{P}(s'|s,a)\left[\mathcal{R}(s,a,s') + \gamma \int_{\mathcal{A}} \pi(a'|s') Q(s',a') da' \right] ds'$, the optimal value function $V^{*}(s) \triangleq \max_{a \in \mathcal{A}} Q^*(s,a)$, and the value function under the policy $\pi$ as $V^\pi(s) \triangleq \max_{a \in \mathcal{A}} Q^\pi(s,a)$.

\subsection{Monte-Carlo Tree Search and Upper Confidence bounds for Trees}
Monte-Carlo Tree Search~(MCTS) is a planning strategy based on a combination of Monte-Carlo sampling and tree search to solve MDPs. MCTS builds a tree where the nodes are the visited states of the MDP, and the edges are the actions executed in each state. MCTS converges to the optimal policy~\citep{kocsis2006improved,xiao2019maximum}, iterating over a loop composed of four steps:
\begin{enumerate}
    \item \textbf{Selection:} starting from the root node, a \textit{tree-policy} is executed to navigate the tree until a node with unvisited children, i.e. expandable node, is reached;
    \item \textbf{Expansion:} the reached node is expanded according to the tree policy;
    \item \textbf{Simulation:} run a rollout, e.g. Monte-Carlo simulation, from the visited child of the current node to the end of the episode;
    \item \textbf{Backup:} use the collected reward to update the action-values $Q(\cdot)$ of the nodes visited in the trajectory from the root node to the expanded node.
\end{enumerate}
The tree-policy used to select the action to execute in each node needs to balance the use of already known good actions, and the visitation of unknown states. The Upper Confidence bounds for Trees~(UCT) sampling strategy~\citep{kocsis2006improved} extends the use of the well-known UCB1 sampling strategy for multi-armed bandits~\citep{auer2002finite}, to MCTS. Considering each node corresponding to a state $s \in \mathcal{S}$ as a different bandit problem, UCT selects an action $a \in \mathcal{A}$ applying an upper bound to the action-value function
\begin{equation}
    \text{UCT}(s,a) = Q(s,a) + \epsilon \sqrt{\dfrac{\log{N(s)}}{N(s,a)}},
\end{equation}
where $N(s,a)$ is the number of executions of action $a$ in state $s$, $N(s) = \sum_a N(s,a)$, and $\epsilon$ is a constant parameter to tune exploration. UCT asymptotically converges to the optimal action-value function $Q^*$, for all states and actions, with the probability of executing a suboptimal action at the root node approaching $0$ with a polynomial rate $O(\frac{1}{t})$, for a simulation budget $t$~\citep{kocsis2006improved,xiao2019maximum}.

\section{Regularized Monte-Carlo Tree Search}
The success of RL methods based on entropy regularization comes from their ability to achieve state-of-the-art performance in decision making and control problems, while enjoying theoretical guarantees and ease of implementation~\citep{haarnoja2018soft,schulman2015trust,lee2018sparse}. However, the use of entropy regularization is MCTS is still mostly unexplored, although its advantageous exploration and value function estimation would be desirable to reduce the detrimental effect of high-branching factor in AlphaGo and AlphaZero. To the best of our knowledge, the MENTS algorithm~\citep{xiao2019maximum} is the first and only method to combine MCTS and entropy regularization. In particular, MENTS uses a maximum entropy regularizer in AlphaGo, proving an exponential convergence rate to the solution of the respective softmax objective function and achieving state-of-the-art performance in some Atari games~\citep{bellemare2013arcade}. In the following, motivated by the success in RL and the promising results of MENTS, we derive a unified theory of regularization in MCTS based on the Legendre-Fenchel transform~\citep{geist2019theory}, that generalizes the use of maximum entropy of MENTS to an arbitrary convex regularizer. Notably, our theoretical framework enables to rigorously motivate the advantages of using maximum entropy and other entropy-based regularizers, such as relative entropy or Tsallis entropy, drawing connections with their RL counterparts TRPO~\citep{schulman2015trust} and Sparse DQN~\citep{lee2018sparse}, as MENTS does with Soft Actor-Critic~(SAC)~\citep{haarnoja2018soft}.
\subsection{Legendre-Fenchel transform}\label{S:leg-fen}
Consider an MDP $\mathcal{M} = \langle \mathcal{S}, \mathcal{A}, \mathcal{R}, \mathcal{P}, \gamma \rangle$, as previously defined. Let $\Omega: \Pi \to \mathbb{R}$ be a strongly convex function. For a policy $\pi_s = \pi(\cdot|s)$ and $Q_s = Q(s,\cdot) \in \mathbb{R}^\mathcal{A}$, the Legendre-Fenchel transform (or convex conjugate) of $\Omega$ is $\Omega^{*}: \mathbb{R}^\mathcal{A} \to \mathbb{R}$, defined as:
\begin{flalign}
\Omega^{*}(Q_s) \triangleq \max_{\pi_s \in \Pi_s} \mathcal{T}_{\pi_s} Q_s - \tau \Omega(\pi_s),\label{E:leg-fen}
\end{flalign}
where the temperature $\tau$ specifies the strength of regularization. Among the several properties of the Legendre-Fenchel transform, we use the following~\citep{mensch2018differentiable,geist2019theory}.
\begin{prop}\label{lb_prop1}
Let $\Omega$ be strongly convex.
\begin{itemize}
\item Unique maximizing argument: $\nabla \Omega^{*}$ is Lipschitz and satisfies
    \begin{flalign}
        \nabla \Omega^{*}(Q_s) = \argmax_{\pi_s \in \Pi_s} \mathcal{T}_{\pi_s} Q_s - \tau \Omega(\pi_s).
    \end{flalign}
\item Boundedness: if there are constants $L_{\Omega}$ and $U_{\Omega}$ such that for all $\pi_s \in \Pi_s$, we have $L_{\Omega} \leq \Omega(\pi_s) \leq U_{\Omega}$, then
    \begin{flalign}
        \max_{a \in \mathcal{A}} Q_s(a) -\tau U_\Omega \leq \Omega^*(Q_s) \leq \max_{a \in \mathcal{A}} Q_s(a) - \tau L_\Omega.
    \end{flalign}
    \item Contraction: for any $Q_1, Q_2 \in \mathbb{R}^{\mathcal{S}\times\mathcal{A}}$
    \begin{flalign}
       \parallel \Omega^{*}(Q_{1}) - \Omega^{*} (Q_{2}) \parallel_{\infty} \leq \gamma \parallel Q_1 - Q_2\parallel_{\infty}.
    \end{flalign}
\end{itemize}
\end{prop}
Note that if $\Omega(\cdot)$ is strongly convex, $\tau\Omega(\cdot)$ is also strongly convex; thus all the properties shown in Proposition 1 still hold\footnote{Other works use the same formula, e.g. Equation (\ref{E:leg-fen}) in~\citet{niculae2017regularized}.}.

Solving equation (2) leads to the solution of the optimal primal policy function
$\nabla \Omega^*(\cdot)$. Since $\Omega(\cdot)$ is strongly convex, the dual function $\Omega^*(\cdot)$ is also convex. One can solve the optimization problem~(\ref{E:leg-fen}) in the dual space~\cite{nachum2020reinforcement} as
\begin{align}
    \Omega(\pi_s) = \max_{Q_s \in \mathbb{R}^\mathcal{A}} \mathcal{T}_{\pi_s} Q_s - \tau \Omega^*(Q_s)
\end{align}
and find the solution of the optimal dual value function as $\Omega^*(\cdot)$. Note that the Legendre-Fenchel transform of the value conjugate function is the convex function $\Omega$, i.e. $\Omega^{**} = \Omega$. In the next section, we leverage on this primal-dual connection based on the Legendre-Fenchel transform as both conjugate value function and policy function, to derive the regularized MCTS backup and tree policy.

\subsection{Regularized backup and tree policy}\label{S:leg-fen-mcts}

In MCTS, each node of the tree represents a state $s \in \mathcal{S}$ and contains a visitation count $N(s,a)$. Given a trajectory, we define $n(s_T)$ as the leaf node corresponding to the reached state $s_T$. Let ${s_0, a_0, s_1, a_1...,s_T}$ be the state action trajectory in a simulation, where $n(s_T)$ is a leaf node of $\mathcal{T}$. Whenever a node $n(s_T)$ is expanded, the respective action values (Equation~\ref{E:leg-fen-backup}) are initialized as $Q_{\Omega}(s_T, a) = 0$, and $N(s_T, a) = 0$ for all $a \in \mathcal{A}$. For all nodes in the trajectory, the visitation count is updated by $N(s_t,a_t) = N(s_t,a_t) + 1$, and the action-values by
\begin{equation}
  Q_{\Omega}(s_t,a_t) =
    \begin{cases}
      r(s_t,a_t) + \gamma \rho & \text{if $t = T$}\\
      r(s_t,a_t) + \gamma \Omega^*(Q_\Omega(s_{t+1})/\tau) & \text{if $t < T$}
    \end{cases}\label{E:leg-fen-backup}
\end{equation}
where $Q_\Omega(s_{t+1}) \in \mathbb{R}^\mathcal{A}$ with $Q_\Omega(s_{t+1},a), \forall a \in \mathcal{A}$, and $\rho$ is an estimate returned from an evaluation function computed in $s_T$, e.g. a discounted cumulative reward averaged over multiple rollouts, or the value-function of node $n(s_{T+1})$ returned by a value-function approximator, e.g.\ a neural network pretrained with deep $Q$-learning~\citep{mnih2015human}, as done in~\citep{silver2016mastering,xiao2019maximum}. We revisit the E2W sampling strategy limited to maximum entropy regularization~\citep{xiao2019maximum} and, through the use of the convex conjugate in Equation~(\ref{E:leg-fen-backup}), we derive a novel sampling strategy that generalizes to any convex regularizer
\begin{flalign}
\pi_t(a_t|s_t) = (1 - \lambda_{s_t}) \nabla \Omega^{*} (Q_{\Omega}(s_t)/\tau)(a_t) + \frac{\lambda_{s_t}}{|\mathcal{A}|}, \label{E:e3w}
\end{flalign}
where $\lambda_{s_t} = \nicefrac{\epsilon |\mathcal{A}|}{\log(\sum_a N(s_t,a) + 1)}$ with $\epsilon > 0$ as an exploration parameter, and $\nabla \Omega^{*}$ depends on the measure in use (see Table~\ref{T:reg-operators} for maximum, relative, and Tsallis entropy).
We call this sampling strategy \textit{Extended Empirical Exponential Weight}~(E3W) to highlight the extension of E2W from maximum entropy to a generic convex regularizer. E3W defines the connection to the duality representation using the Legendre-Fenchel transform, that is missing in E2W. Moreover, while the Legendre-Fenchel transform can be used to derive a theory of several state-of-the-art algorithms in RL, such as TRPO, SAC, A3C~\cite{geist:l1pbr}, our result is the first introducing the connection with MCTS.

\subsection{Convergence rate to regularized objective}
We show that the regularized value $V_{\Omega}$ can be effectively estimated at the root state $s \in \mathcal{S}$, with the assumption that each node in the tree has a $\sigma^{2}$-subgaussian distribution. This result extends the analysis provided in~\citep{xiao2019maximum}, which is limited to the use of maximum entropy.\\ 
\begin{manualtheorem}{1}
At the root node $s$ where $N(s)$ is the number of visitations, with $\epsilon > 0$, $V_{\Omega}(s)$ is the estimated value, with constant $C$ and $\hat{C}$, we have
\begin{flalign}
\mathbb{P}(| V_{\Omega}(s) - V^{*}_{\Omega}(s) | > \epsilon) \leq C \exp\{\frac{-N(s) \epsilon}{\hat{C} \sigma \log^2(2 + N(s))}\},
\end{flalign}
\end{manualtheorem}
where $V_{\Omega}(s) = \Omega^*(Q_s)$ and $V^{*}_{\Omega}(s) = \Omega^*(Q^*_s)$.

From this theorem, we obtain that the convergence rate of choosing the best action $a^*$ at the root node, when using the E3W strategy, is exponential.
\begin{manualtheorem}{2}
Let $a_t$ be the action returned by E3W at step $t$. For large enough $t$ and constants $C, \hat{C}$
\begin{flalign}
&\mathbb{P}(a_t \neq a^{*}) \leq C t\exp\{-\frac{t}{\hat{C} \sigma (\log(t))^3}\}.
\end{flalign}
\end{manualtheorem}
This result shows that, for every strongly convex regularizer, the convergence rate of choosing the best action at the root node is exponential, as already proven in the specific case of maximum entropy~\cite{xiao2019maximum}.
\section{Entropy-regularization backup operators}\label{S:algs}
From the introduction of a unified view of generic strongly convex regularizers as backup operators in MCTS, we narrow the analysis to entropy-based regularizers. For each entropy function, Table~\ref{T:reg-operators} shows the Legendre-Fenchel transform and the maximizing argument, which can be respectively replaced in our backup operation~(Equation~\ref{E:leg-fen-backup}) and sampling strategy E3W~(Equation~\ref{E:e3w}). Using maximum entropy retrieves the maximum entropy MCTS problem introduced in the MENTS algorithm~\citep{xiao2019maximum}. This approach closely resembles the maximum entropy RL framework used to encourage exploration~\citep{haarnoja2018soft,schulman2017equivalence}.
We introduce two novel MCTS algorithms based on the minimization of relative entropy of the policy update, inspired by trust-region~\citep{schulman2015trust} and proximal optimization methods~\citep{schulman2017proximal} in RL, and on the maximization of Tsallis entropy, which has been more recently introduced in RL as an effective solution to enforce the learning of sparse policies~\citep{lee2018sparse}. We call these algorithms RENTS and TENTS. Contrary to maximum and relative entropy, the definition of the Legendre-Fenchel and maximizing argument of Tsallis entropy is non-trivial, being
\begin{flalign}
\Omega^*(Q_t) &= \tau \cdot \text{spmax}(Q_t(s,\cdot)/\tau\label{E:leg-fen-tsallis}),\\
\nabla \Omega^{*}(Q_t) &= \max \lbrace\frac{Q_t(s,a)}{\tau} - \frac{\sum_{a \in \mathcal{K}} Q_t(s,a)/\tau - 1}{|\mathcal{K}|}, 0 \rbrace,\label{E:max-arg-tsallis}
\end{flalign}
where spmax is defined for any function $f:\mathcal{S} \times \mathcal{A} \rightarrow \mathbb{R}$ as
\begin{flalign}
\text{spmax}&(f(s,\cdot)) \triangleq \\\nonumber& \sum_{a \in \mathcal{K}} \Bigg( \frac{f(s,a)^2}{2} - \frac{(\sum_{a \in \mathcal{K}} f(s,a) - 1)^2}{2|\mathcal{K}|^2} \Bigg) + \frac{1}{2},
\end{flalign}
and $\mathcal{K}$ is the set of actions that satisfy $1 + if(s,a_i) > \sum_{j=1}^{i}f(s,a_j)$, with $a_i$ indicating the action with the $i$-th largest value of $f(s, a)$~\citep{lee2018sparse}.

\begin{table*}[ht]
\caption{List of entropy regularizers with Legendre-Fenchel transforms and maximizing arguments.}
\centering
\renewcommand*{\arraystretch}{2}
\begin{tabular}{cccc} 
    \toprule
    \textbf{Entropy} &  \textbf{Regularizer} $\Omega(\pi_s)$ & \textbf{Legendre-Fenchel} $\Omega^*(Q_s)$ & \textbf{Max argument} $\nabla\Omega^*(Q_s)$\\ \hline
    \midrule
 Maximum & $\sum_a \pi(a|s) \log \pi(a|s)$ & $\tau \log\sum_a e^{\frac{Q(s,a)}{\tau}}$ &  $\dfrac{e^{\frac{Q(s,a)}{\tau}}}{\sum_b e^{\frac{Q(s,b)}{\tau}}}$\\
 \hline
 Relative & $\text{D}_{\text{KL}}(\pi_{t}(a|s) || \pi_{t-1}(a|s))$ & $\tau \log\sum_a\pi_{t-1}(a|s) e^{\frac{Q_t(s,a)}{\tau}}$ & $\dfrac{\pi_{t-1}(a|s)e^{\frac{Q_t(s,a)}{\tau}}}{\sum_b\pi_{t-1}(b|s)e^{\frac{Q_t(s,b)}{\tau}}}$\\
 \hline
 Tsallis & $\frac{1}{2} (\parallel \pi(a|s) \parallel^2_2 - 1)$ & Equation~(\ref{E:leg-fen-tsallis}) & Equation~(\ref{E:max-arg-tsallis})\\
 \hline
\bottomrule
\end{tabular}\label{T:reg-operators}
\end{table*}

\subsection{Regret analysis}
At the root node, let each children node $i$ be assigned with a random variable $X_i$, with mean value $V_i$, while the quantities related to the optimal branch are denoted by $*$, e.g. mean value $V^{*}$.
At each timestep $n$, the mean value of variable $X_i$ is $V_{i_n}$. 
The pseudo-regret~\citep{coquelin2007bandit} at the root node, at timestep $n$, is defined as $R^{\text{UCT}}_n = nV^{*} - \sum^{n}_{t = 1}V_{i_t}$.
Similarly, we define the regret of E3W at the root node of the tree as
\begin{align}\label{regret}
R_n = nV^{*} - \sum_{t=1}^n V_{i_t} &= nV^{*} - \sum_{t=1}^n \mathbb{I} (i_t = i) V_{i_t} \\\nonumber &= nV^{*} - \sum_i  V_i \sum_{t=1}^n  \hat{\pi}_t(a_i|s),
\end{align}
where $\hat{\pi}_t(\cdot)$ is the policy at time step $t$, and $\mathbb{I}(\cdot)$ is the indicator function.\\
The expected regret is defined as
\begin{align}
\mathbb{E}[R_n] = nV^{*} - \sum_{t=1}^n  \left\langle \hat{\pi}_t(\cdot), V(\cdot)\right\rangle.
\end{align}
\begin{manualtheorem}{3}
Consider an E3W policy applied to the tree. Let define $\mathcal{D}_{\Omega^*}(x,y) = \Omega^*(x) - \Omega^*(y) - \nabla \Omega^* (y) (x - y)$ as the Bregman divergence between $x$ and $y$, The expected pseudo regret $R_n$ satisfies
\begin{flalign}
\mathbb{E}[R_n] \leq& - \tau \Omega(\hat{\pi}) + \sum_{t=1}^n \mathcal{D}_{\Omega^*} (\hat{V_t}(\cdot) +V(\cdot), \hat{V_t}(\cdot))\\ &+ \mathcal{O} (\frac{n}{\log n})\nonumber.
\end{flalign}
\end{manualtheorem}
This theorem bounds the regret of E3W for a generic convex regularizer $\Omega$; the regret bounds for each entropy regularizer can be easily derived from it. Let $m = \min_{a} \nabla \Omega^{*}(a|s)$.

\begin{corollary}Maximum entropy regret:\\ 
$   \mathbb{E}[R_n] \leq \tau (\log |\mathcal{A}|) + \frac{n|\mathcal{A}|}{\tau} + \mathcal{O} (\frac{n}{\log n})\nonumber$.
\end{corollary}
\begin{corollary}Relative entropy regret:\\
$   \mathbb{E}[R_n] \leq \tau (\log |\mathcal{A}| - \frac{1}{m}) + \frac{n|\mathcal{A}|}{\tau} + \mathcal{O} (\frac{n}{\log n})\nonumber$.
\end{corollary}
\begin{corollary}Tsallis entropy regret:\\
$\mathbb{E}[R_n] \leq \tau (\frac{|\mathcal{A}| - 1}{|\mathcal{A}|}) + \frac{n|\mathcal{K}|}{2} + \mathcal{O} (\frac{n}{\log n})\nonumber$.\label{cor:regret_tsallis}
\end{corollary}

\paragraph{Remarks.} The regret bound of UCT and its variance have already been analyzed for non-regularized MCTS with binary tree~\citep{coquelin2007bandit}. On the contrary, our regret bound analysis in Theorem~3 applies to generic regularized MCTS. From the specialized bounds in the corollaries, we observe that the maximum and relative entropy share similar results, although the bounds for relative entropy are slightly smaller due to $\frac{1}{m}$. Remarkably, the bounds for Tsallis entropy become tighter for increasing number of actions, which translates in limited regret in problems with high branching factor. This result establishes the advantage of Tsallis entropy in complex problems w.r.t. to other entropy regularizers, as empirically confirmed in Section~\ref{S:exps}.

\begin{figure*}
\centering
\includegraphics[scale=.5]{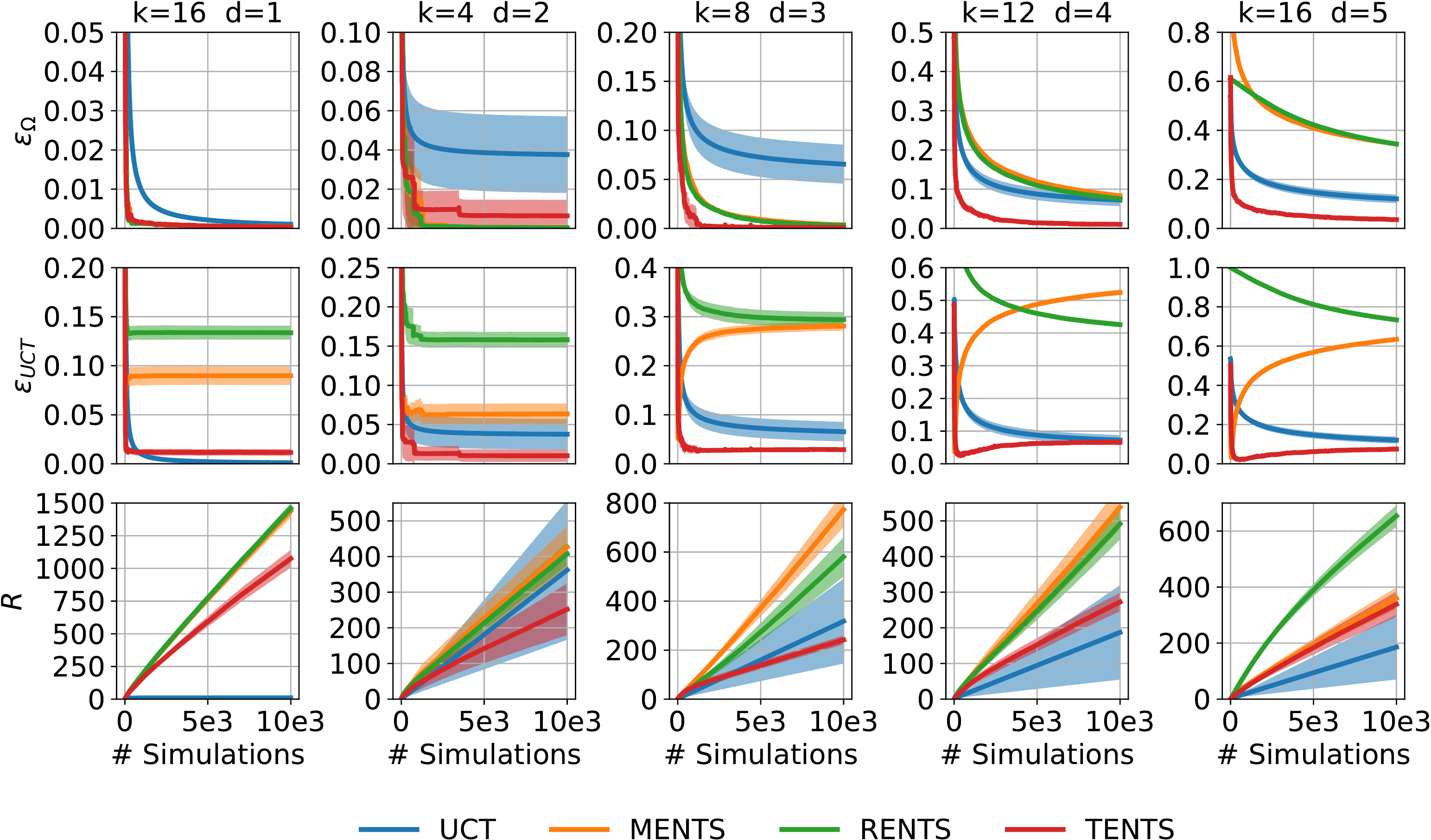}
\caption{For each algorithm, we show the convergence of the value estimate at the root node to the respective optimal value (top), to the UCT optimal value (middle), and the regret (bottom).}
\label{F:synth_plots}
\end{figure*}

\subsection{Error analysis}
We analyse the error of the regularized value estimate at the root node $n(s)$ w.r.t. the optimal value: $\varepsilon_{\Omega} = V_{\Omega}(s) - V^{*}(s)$.

\begin{manualtheorem}{4}
For any $\delta > 0$ and generic convex regularizer $\Omega$, with some constant $C, \hat{C}$, with probability at least $1 - \delta$, $\varepsilon_{\Omega}$ satisfies
\begin{flalign}
&-\sqrt{\frac{\Hat{C}\sigma^2\log\frac{C}{\delta}}{2N(s)}} - \frac{\tau(U_{\Omega} - L_{\Omega})}{1 - \gamma} \leq \varepsilon_{\Omega}  \leq \sqrt{\frac{\Hat{C}\sigma^2\log\frac{C}{\delta}}{2N(s)}}.\label{E:error}
\end{flalign}
\end{manualtheorem}
To the best of our knowledge, this theorem provides the first result on the error analysis of value estimation at the root node of convex regularization in MCTS. To give a better understanding of the effect of each entropy regularizer in Table~\ref{T:reg-operators}, we specialize the bound in Equation~\ref{E:error} to each of them. From~\citep{lee2018sparse}, we know that for maximum entropy $\Omega(\pi_t) = \sum_a \pi_t \log \pi_t$, we have $-\log |\mathcal{A}| \leq \Omega(\pi_t) \leq 0$; for relative entropy $\Omega(\pi_t) = \text{KL}(\pi_t || \pi_{t-1})$, if we define $m = \min_{a} \pi_{t-1}(a|s)$, then we can derive $0 \leq \Omega(\pi_t) \leq -\log |\mathcal{A}| + \log \frac{1}{m}$; and for Tsallis entropy $\Omega(\pi_t) = \frac{1}{2} (\parallel \pi_t \parallel^2_2 - 1)$, we have $- \frac{|\mathcal{A}| - 1}{2|\mathcal{A}|} \leq \Omega(\pi_t) \leq 0$. Then, defining $\Psi = \sqrt{\frac{\Hat{C}\sigma^2\log\frac{C}{\delta}}{2N(s)}}$,
\begin{corollary}Maximum entropy error: \\
$-\Psi - \dfrac{\tau \log |\mathcal{A}|}{1 - \gamma} \leq \varepsilon_{\Omega}  \leq \Psi$.
\end{corollary}
\begin{corollary}Relative entropy error: \\
$-\Psi - \dfrac{\tau(\log |\mathcal{A}| - \log \frac{1}{m})}{1 - \gamma} \leq \varepsilon_{\Omega}  \leq \Psi$.
\end{corollary}
\begin{corollary}Tsallis entropy error: \\
$-\Psi - \dfrac{|\mathcal{A}| - 1}{2|\mathcal{A}|} \dfrac{\tau}{1 - \gamma} \leq \varepsilon_{\Omega}  \leq \Psi$.\label{cor:tsallis}
\end{corollary}
These results show that when the number of actions $|\mathcal{A}|$ is large, TENTS enjoys the smallest error; moreover, we also see that lower bound of RENTS is always smaller than for MENTS.

\section{Empirical evaluation}\label{S:exps}
In this section, we empirically evaluate the benefit of the proposed entropy-based MCTS regularizers. First, we complement our theoretical analysis with an empirical study of the synthetic tree toy problem introduced in~\citet{xiao2019maximum}, which serves as a simple scenario to give an interpretable demonstration of the effects of our theoretical results in practice. Second, we compare to AlphaGo~\citep{silver2016mastering}, recently introduced to enable MCTS to solve large scale problems with high branching factor. Our implementation is a simplified version of the original algorithm, where we remove various tricks in favor of better interpretability. For the same reason, we do not compare with the most recent and state-of-the-art MuZero~\citep{schrittwieser2019mastering}, as this is a slightly different solution highly tuned to maximize performance, and a detailed description of its implementation is not available.

\begin{figure*}
\centering
\subfigure[\label{F:heat_reg}]{\includegraphics[scale=.5]{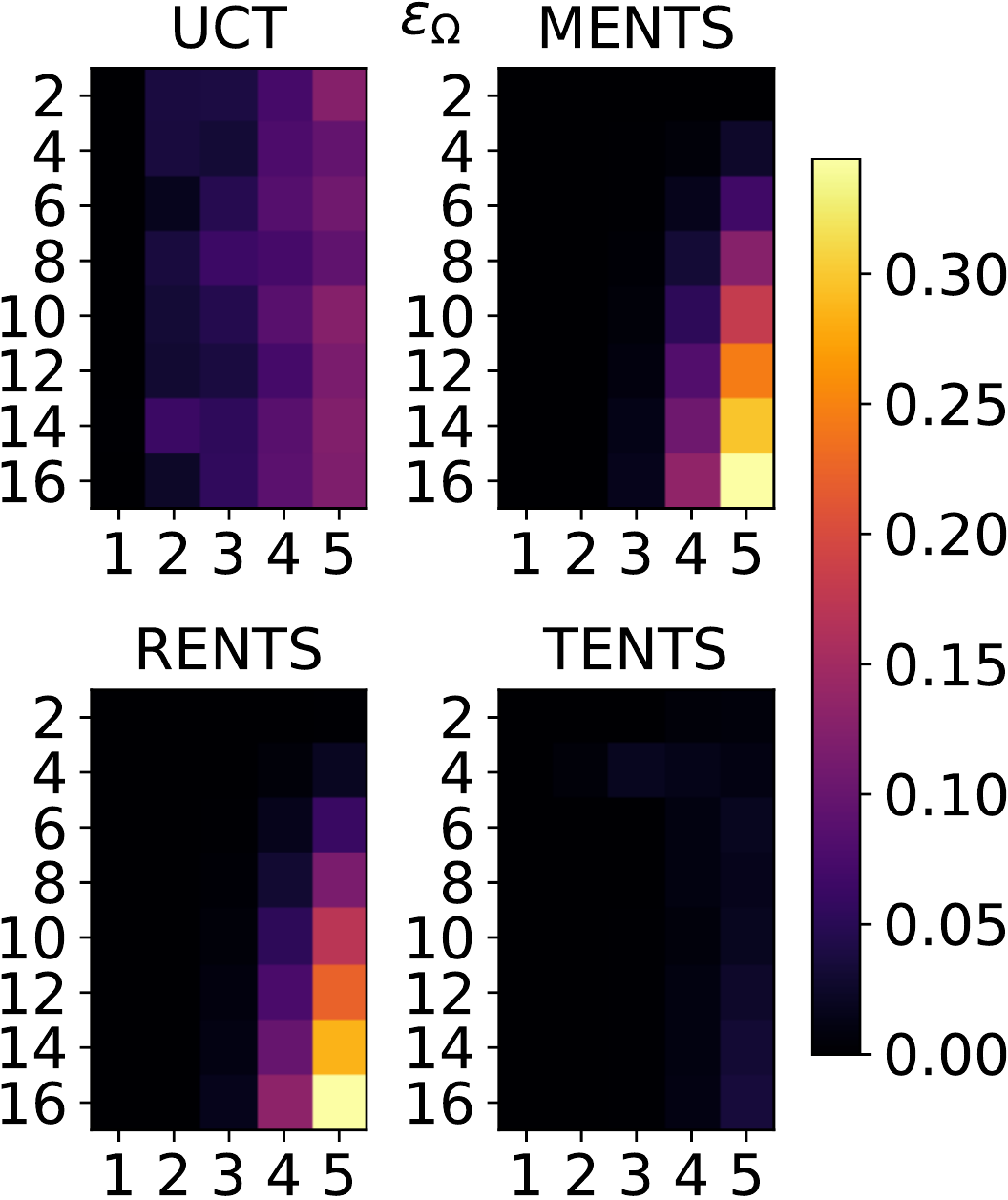}}
\subfigure[\label{F:heat_uct}]{\includegraphics[scale=.5]{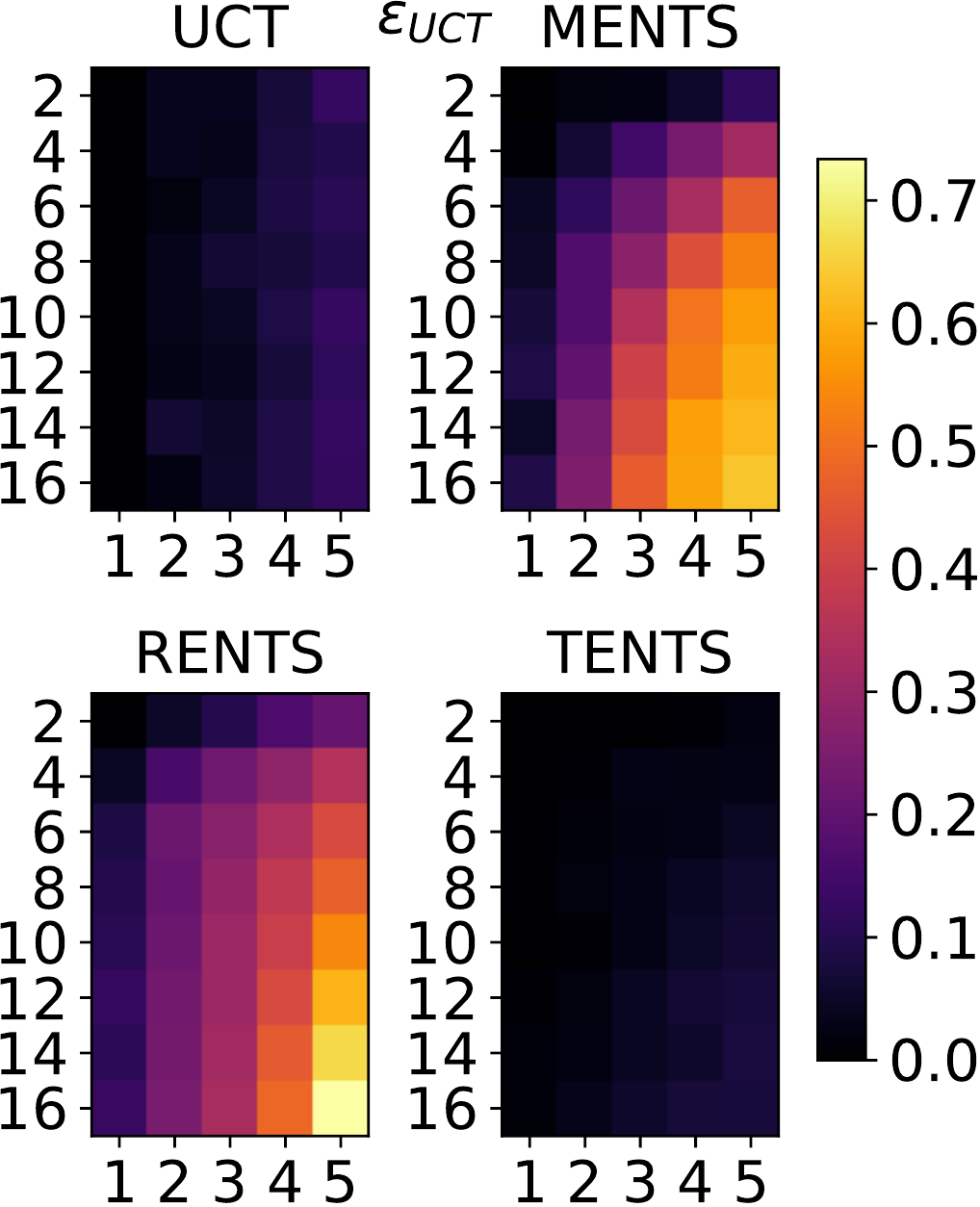}}
\subfigure[\label{F:heat_regret}]{\includegraphics[scale=.5]{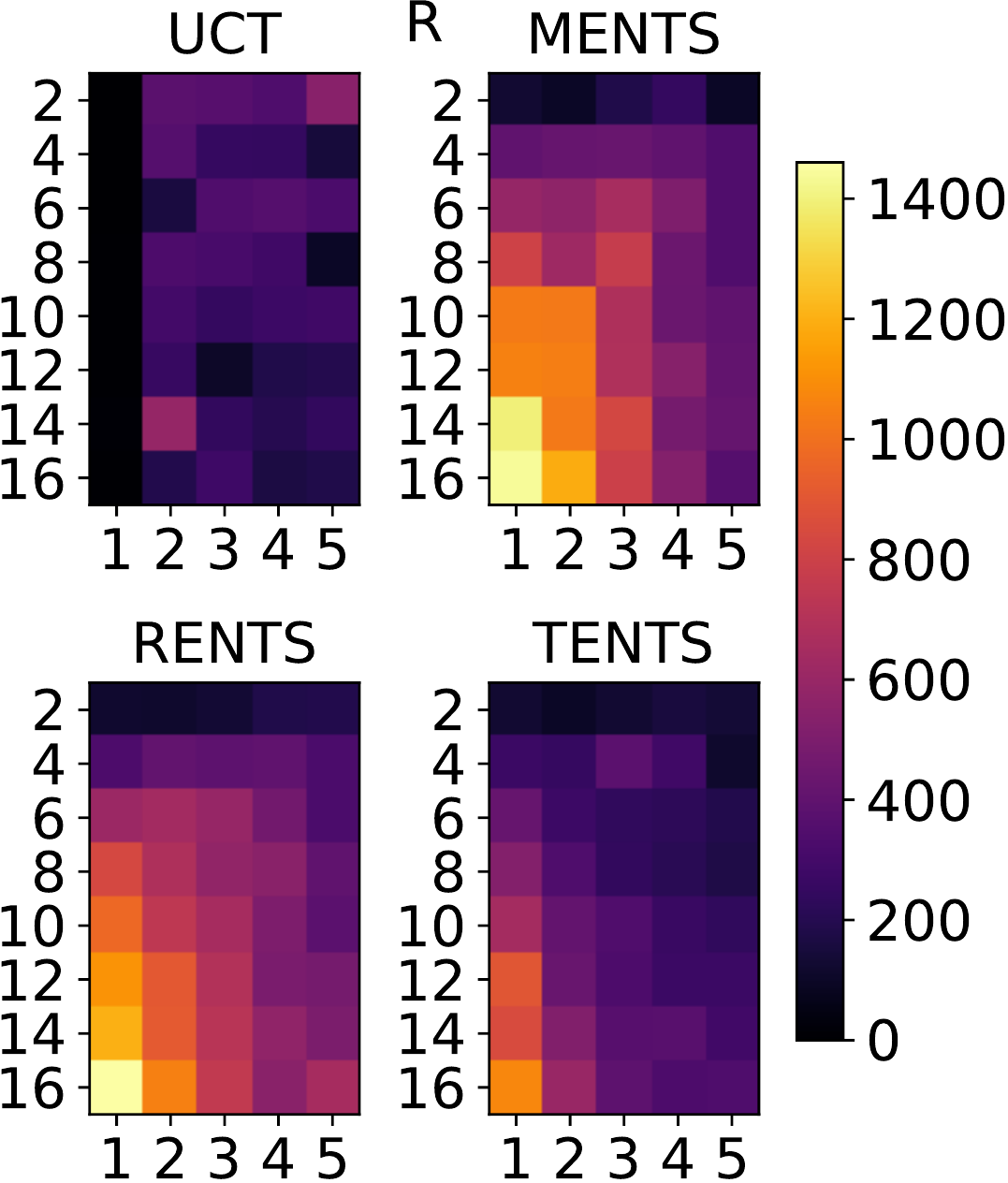}}
\caption{For different branching factor $k$ (rows) and depth $d$ (columns), the heatmaps show: the absolute error of the value estimate at the root node after the last simulation of each algorithm w.r.t. the respective optimal value (a), and w.r.t. the optimal value of UCT (b); regret at the root node (c).}
\label{F:heatmaps}
\end{figure*}

\subsection{Synthetic tree}
This toy problem is introduced in~\citet{xiao2019maximum} to highlight the improvement of MENTS over UCT. It consists of a tree with branching factor $k$ and depth $d$. Each edge of the tree is assigned a random value between $0$ and $1$. At each leaf, a Gaussian distribution is used as an evaluation function resembling the return of random rollouts. The mean of the Gaussian distribution is the sum of the values assigned to the edges connecting the root node to the leaf, while the standard deviation is $\sigma=0.05$\footnote{The value of the standard deviation is not provided in~\citet{xiao2019maximum}. After trying different values, we observed that our results match the one in~\citet{xiao2019maximum} when using $\sigma=0.05$.}. For stability, all the means are normalized between $0$ and $1$. As in~\citet{xiao2019maximum}, we create $5$ trees on which we perform $5$ different runs in each, resulting in $25$ experiments, for all the combinations of branching factor $k = \lbrace 2, 4, 6, 8, 10, 12, 14, 16 \rbrace$ and depth $d = \lbrace 1, 2, 3, 4, 5 \rbrace$, computing: (i) the value estimation error at the root node w.r.t. the regularized optimal value: $\varepsilon_\Omega = V_\Omega - V^*_\Omega$; (ii) the value estimation error at the root node w.r.t. the unregularized optimal value: $\varepsilon_\text{UCT} = V_\Omega - V^*_{\text{UCT}}$; (iii) the regret $R$ as in Equation~(\ref{regret}). For a fair comparison, we use fixed $\tau = 0.1$ and $\epsilon=0.1$ across all algorithms. Figure~\ref{F:synth_plots} and~\ref{F:heatmaps} show how UCT and each regularizer behave for different configurations of the tree. We observe that, while RENTS and MENTS converge slower for increasing tree sizes, TENTS is robust w.r.t. the size of the tree and almost always converges faster than all other methods to the respective optimal value. Notably, the optimal value of TENTS seems to be very close to the one of UCT, i.e. the optimal value of the unregularized objective, and also converges faster than the one estimated by UCT, while MENTS and RENTS are considerably further from this value. In terms of regret, UCT explores less than the regularized methods and it is less prone to high regret, at the cost of slower convergence time. Nevertheless, the regret of TENTS is the smallest between the ones of the other regularizers, which seem to explore too much. These results show a general superiority of TENTS in this toy problem, also confirming our theoretical findings about the advantage of TENTS in terms of approximation error~(Corollary~\ref{cor:tsallis}) and regret~(Corollary~\ref{cor:regret_tsallis}), in problems with many actions.

\subsection{Entropy-regularized AlphaGo}
\begin{table*}[t]
\fontsize{9}{8.5}\selectfont
\caption{Average score in Atari over $100$ seeds per game. Bold denotes no statistically significant difference to the highest mean (t-test, $p < 0.05$). Bottom row shows \# no difference to highest mean.}
\label{pocman_table}
\centering
\smallskip
\renewcommand*{\arraystretch}{1.5}
\begin{tabular}{lcccccc}\hline
 & $\text{UCT}$ & $\text{MaxMCTS}$ & $\text{MENTS}$ & $\text{RENTS}$ & $\text{TENTS}$ \\
\toprule
\cline{1-6}
Alien& $\mathbf{1,486.80}$& $\mathbf{1,461.10}$& $\mathbf{1,508.60}$& $\mathbf{1,547.80}$& $\mathbf{1,568.60}$\\ \cline{1-6}
Amidar& $115.62$& $\mathbf{124.92}$& $\mathbf{123.30}$& $\mathbf{125.58}$& $\mathbf{121.84}$\\ \cline{1-6}
Asterix& $4,855.00$& $\mathbf{5,484.50}$& $\mathbf{5,576.00}$& $\mathbf{5,743.50}$& $\mathbf{5,647.00}$\\ \cline{1-6}
Asteroids& $873.40$& $899.60$& $1,414.70$& $1,486.40$& $\mathbf{1,642.10}$\\ \cline{1-6}
Atlantis& $35,182.00$& $\mathbf{35,720.00}$& $\mathbf{36,277.00}$& $35,314.00$& $\mathbf{35,756.00}$\\ \cline{1-6}
BankHeist& $475.50$& $458.60$& $\mathbf{622.30}$& $\mathbf{636.70}$& $\mathbf{631.40}$\\ \cline{1-6}
BeamRider& $\mathbf{2,616.72}$& $\mathbf{2,661.30}$& $\mathbf{2,822.18}$& $2,558.94$& $\mathbf{2,804.88}$\\ \cline{1-6}
Breakout& $\mathbf{303.04}$& $296.14$& $\mathbf{309.03}$& $300.35$& $\mathbf{316.68}$\\ \cline{1-6}
Centipede& $1,782.18$& $1,728.69$& $\mathbf{2,012.86}$& $\mathbf{2,253.42}$& $\mathbf{2,258.89}$\\ \cline{1-6}
DemonAttack& $579.90$& $640.80$& $\mathbf{1,044.50}$& $\mathbf{1,124.70}$& $\mathbf{1,113.30}$\\ \cline{1-6}
Enduro& $\mathbf{129.28}$& $124.20$& $128.79$& $\mathbf{134.88}$& $\mathbf{132.05}$\\ \cline{1-6}
Frostbite& $1,244.00$& $1,332.10$& $\mathbf{2,388.20}$& $\mathbf{2,369.80}$& $\mathbf{2,260.60}$\\ \cline{1-6}
Gopher& $3,348.40$& $3,303.00$& $\mathbf{3,536.40}$& $\mathbf{3,372.80}$& $\mathbf{3,447.80}$\\ \cline{1-6}
Hero& $3,009.95$& $3,010.55$& $\mathbf{3,044.55}$& $\mathbf{3,077.20}$& $\mathbf{3,074.00}$\\ \cline{1-6}
MsPacman& $1,940.20$& $1,907.10$& $2,018.30$& $\mathbf{2,190.30}$& $\mathbf{2,094.40}$\\ \cline{1-6}
Phoenix& $2,747.30$& $2,626.60$& $3,098.30$& $2,582.30$& $\mathbf{3,975.30}$\\ \cline{1-6}
Qbert& $7,987.25$& $8,033.50$& $8,051.25$& $8,254.00$& $\mathbf{8,437.75}$\\ \cline{1-6}
Robotank& $\mathbf{11.43}$& $11.00$& $\mathbf{11.59}$& $\mathbf{11.51}$& $\mathbf{11.47}$\\ \cline{1-6}
Seaquest& $\mathbf{3,276.40}$& $\mathbf{3,217.20}$& $\mathbf{3,312.40}$& $\mathbf{3,345.20}$& $\mathbf{3,324.40}$\\ \cline{1-6}
Solaris& $895.00$& $923.20$& $\mathbf{1,118.20}$& $\mathbf{1,115.00}$& $\mathbf{1,127.60}$\\ \cline{1-6}
SpaceInvaders& $778.45$& $\mathbf{835.90}$& $\mathbf{832.55}$& $\mathbf{867.35}$& $\mathbf{822.95}$\\ \cline{1-6}
WizardOfWor& $685.00$& $666.00$& $\mathbf{1,211.00}$& $\mathbf{1,241.00}$& $\mathbf{1,231.00}$\\ \cline{1-6}
\bottomrule
\textbf{\# Highest mean} & $6/22$& $7/22$& $17/22$& $16/22$& $\textbf{22/22}$\\\cline{1-6}
\bottomrule
\end{tabular}\label{T:atari}
\end{table*}

The learning time of AlphaZero can be slow in problems with high branching factor, due to the need of a large number of MCTS simulations for obtaining good estimates of the randomly initialized action-values. To overcome this problem, AlphaGo~\citep{silver2016mastering} initializes the action-values using the values retrieved from a pretrained network, which is kept fixed during the training.

\textbf{Atari.} Atari 2600~\citep{bellemare2013arcade} is a popular benchmark for testing deep RL methodologies~\citep{mnih2015human,van2016deep,bellemare2017distributional} but still relatively disregarded in MCTS. We use a Deep $Q$-Network, pretrained using the same experimental setting of~\citet{mnih2015human}, to initialize the action-value function of each node after expansion as $Q_{init}(s,a) = \left(Q(s,a) - V(s)\right)/\tau$, for MENTS and TENTS, as done in~\citet{xiao2019maximum}. For RENTS we init $Q_{init}(s,a) = \log P_{\text{prior}}(a|s)) + \left(Q(s,a) - V(s)\right)/\tau$, where $P_{\text{prior}}$ is the Boltzmann distribution induced by action-values $Q(s,.)$ computed from the network. Each experimental run consists of $512$ MCTS simulations. The temperature $\tau$ is optimized for each algorithm and game via grid-search between $0.01$ and $1$. The discount factor is $\gamma = 0.99$, and for PUCT the exploration constant is $c = 0.1$. Table~\ref{T:atari} shows the performance, in terms of cumulative reward, of standard AlphaGo with PUCT and our three regularized versions, on $22$ Atari games. Moreover, we test also AlphaGo using the MaxMCTS backup~\citep{khandelwal2016analysis} for further comparison with classic baselines. We observe that regularized MCTS dominates other baselines, in particular TENTS achieves the highest scores in all the $22$ games, showing that sparse policies are more effective in Atari. In particular, TENTS significantly outperforms the other methods in the games with many actions, e.g. Asteroids, Phoenix, confirming the results obtained in the synthetic tree experiment, explained by corollaries~\ref{cor:regret_tsallis} and~\ref{cor:tsallis} on the benefit of TENTS in problems with high-branching factor.

\section{Related Work}
Entropy regularization is a common tool for controlling exploration in
Reinforcement Learning (RL) and has lead to several successful methods
\citep{schulman2015trust,haarnoja2018soft,schulman2017equivalence,mnih2016asynchronous}. Typically specific
forms of entropy are utilized such as maximum entropy
\citep{haarnoja2018soft} or relative
entropy~\citep{schulman2015trust}. This approach is an instance of the more generic duality framework, commonly used in convex optimization theory. Duality has been extensively studied in game theory~\citep{shalev2006convex,pavel2007duality} and more recently in RL, for instance considering mirror descent optimization~\citep{montgomery2016guided,mei2019principled}, drawing the connection between MCTS and regularized policy optimization~\citep{grill2020monte}, or formalizing the RL objective via Legendre-Rockafellar duality~\citep{nachum2020duality}. Recently~\citep{geist2019theory}
introduced regularized Markov Decision Processes, formalizing the RL objective with a generalized form of convex regularization, based on the Legendre-Fenchel transform. In this paper, we provide a novel study of convex regularization in MCTS, and derive relative entropy (KL-divergence) and Tsallis entropy regularized MCTS algorithms, i.e. RENTS and TENTS respectively. Note that the recent maximum entropy MCTS algorithm MENTS~\citep{xiao2019maximum} is a
special case of our generalized regularized MCTS. Unlike MENTS, RENTS
can take advantage of any action distribution prior, in the
experiments the prior is derived using Deep $Q$-learning~\citep{mnih2015human}. On the other
hand, TENTS allows for sparse action exploration and thus higher dimensional
action spaces compared to MENTS.

Several works focus on modifying classical MCTS to improve
exploration. UCB1-tuned \citep{auer2002finite} modifies the upper
confidence bound of UCB1 to account for variance in order to improve
exploration. \citep{tesauro2012bayesian} proposes a Bayesian version of
UCT, which obtains better estimates of node values and uncertainties
given limited experience. Many heuristic approaches based on specific
domain knowledge have been proposed, such as adding a bonus term to
value estimates~\citep{gelly2006exploration,teytaud2010huge,childs2008transpositions,kozelek2009methods,chaslot2008progressive}
or prior knowledge collected during policy
search~\citep{gelly2007combining,helmbold2009all,lorentz2010improving,tom2010investigating,hoock2010intelligent}.
\citep{khandelwal2016analysis} formalizes and analyzes different
on-policy and off-policy complex backup approaches for MCTS planning
based on RL techniques. \citep{vodopivec2017monte} proposes an approach
called SARSA-UCT, which performs the dynamic programming backups using
SARSA~\citep{rummery1995problem}. Both \citep{khandelwal2016analysis}
and \citep{vodopivec2017monte} directly borrow value backup ideas from
RL to estimate the value at each tree node, but they do not provide any proof of convergence.

\section{Conclusion}
We introduced a theory of convex regularization in Monte-Carlo Tree Search~(MCTS) based on the Legendre-Fenchel transform. We proved that a generic strongly convex regularizer has an exponential convergence rate for the selection of the optimal action at the root node. Our result gives theoretical motivations to previous results specific to maximum entropy regularization. Furthermore, we provided the first study of the regret of MCTS when using a generic strongly convex regularizer, and an analysis of the error between the regularized value estimate at the root node and the optimal regularized value. We use these results to motivate the use of entropy regularization in MCTS, considering maximum, relative, and Tsallis entropy, and we specialized our regret and approximation error bounds to each entropy-regularizer. We tested our regularized MCTS algorithm in a simple toy problem, where we give an empirical evidence of the effect of our theoretical bounds for the regret and approximation error. Finally, we introduced the use of convex regularization in AlphaGo, and carried out experiments on several Atari games. Overall, our empirical results show the advantages of convex regularization, and in particular the superiority of Tsallis entropy w.r.t. other entropy-regularizers.
\bibliography{conv_reg}
\bibliographystyle{icml2021}

\input{appendix}

\end{document}

%% file: math_commands.tex
\usepackage{amsmath,amsfonts,bm}

\def\eqref#1{equation~\ref{#1}}

\def\1{\bm{1}}

\DeclareMathAlphabet{\mathsfit}{\encodingdefault}{\sfdefault}{m}{sl}
\SetMathAlphabet{\mathsfit}{bold}{\encodingdefault}{\sfdefault}{bx}{n}

\DeclareMathOperator*{\argmax}{arg\,max}

%% file: appendix.tex
\appendix
\onecolumn

\section{Proofs}\label{A:proofs}
In this section, we describe how to derive the theoretical results presented in the paper.

First, the exponential convergence rate of the estimated value function to the conjugate regularized value function at the root node (Theorem 1) is derived based on induction with respect to the depth $D$ of the tree. When $D = 1$, we derive the concentration of the average reward at the leaf node with respect to the $\infty$-norm (as shown in Lemma 1) based on the result from Theorem 2.19 in \cite{wainwright2019high}, and the induction is done over the tree by additionally exploiting the contraction property of the convex regularized value function. Second, based on Theorem 1, we prove the exponential convergence rate of choosing the best action at the root node (Theorem 2). Third, the pseudo-regret analysis of E3W is derived based on the Bregman divergence properties and the contraction properties of the Legendre-Fenchel transform (Proposition 1). Finally, the bias error of estimated value at the root node is derived based on results of Theorem 1, and the boundedness property of the Legendre-Fenchel transform (Proposition 1).

Let $\hat{r}$ and $r$ be respectively the average and the the expected reward at the leaf node, and the reward distribution at the leaf node be $\sigma^2$-sub-Gaussian.
\begin{lemma}
For the stochastic bandit problem E3W guarantees that, for $t\geq 4$,
\begin{flalign}
\mathbb{P}\big( \parallel r - \hat{r}_t\parallel_{\infty} \geq \frac{2\sigma}{\log(2 + t)}\big) \leq 4 |\mathcal{A}| \exp\Big(-\frac{t}{(\log (2 + t))^3}\Big). \nonumber
\end{flalign}
\end{lemma}

\begin{proof}
Let us define $N_t(a)$ as the number of times action $a$ have been chosen until time $t$, and $\hat{N_t}(a) = \sum^{t}_{s=1} \pi_{s}(a)$, where $\pi_{s}(a)$ is the E3W policy at time step $s$. By choosing $\lambda_s = \frac{|\mathcal{A}|}{\log (1 + s)}$, it follows that for all $a$ and $t \geq 4$,
\begin{flalign}
\hat{N_t}(a) &= \sum^{t}_{s=1} \pi_{s}(a) \geq \sum^{t}_{s=1}\frac{1}{\log(1+s)} \geq \sum^{t}_{s=1} \frac{1}{\log(1 + s)} - \frac{s/(s+1)}{(\log(1 + s))^2} \nonumber\\
&\geq \int_1^{1+t} \frac{1}{\log(1+s)} - \frac{s/(s+1)}{(\log(1 + s))^2}ds = \frac{1+t}{\log(2+t)} - \frac{1}{\log 2} \geq \frac{t}{2\log(2+t)}\nonumber.
\end{flalign}
From Theorem 2.19 in~\cite{wainwright2019high}, we have the following concentration inequality:
\begin{flalign}
\mathbb{P} (|N_t(a) - \hat{N}_t(a)| > \epsilon) \leq 2 \exp\{-\frac{\epsilon^2}{2\sum_{s=1}^t \sigma_s^2}\} \leq 2\exp\{-\frac{2\epsilon^2}{t}\} \nonumber,
\end{flalign}
where $\sigma^2_s \leq 1/4$ is the variance of a Bernoulli distribution with $p = \pi_s(k)$ at time step s. We define the event
\begin{flalign}
E_{\epsilon} = \{\forall a \in \mathcal{A}, |\hat{N_t}(a) - N_t(a)| \leq \epsilon\} \nonumber,
\end{flalign}
and consequently
\begin{flalign}
\mathbb{P} (|\hat{N_t}(a) - N_t(a)| \geq \epsilon) \leq 2|\mathcal{A}|\exp(-\frac{2\epsilon^2}{t}). \label{lb_subgaussion} 
\end{flalign}
Conditioned on the event $E_{\epsilon}$, for $\epsilon = \frac{t}{4\log(2+t)}$, we have $N_t(a) \geq \frac{t}{4\log(2+t)}$. 
For any action a by the definition of sub-gaussian,
\begin{flalign}
&\mathbb{P}\Bigg( |r(a) - \hat{r}_t(a)| > \sqrt{\frac{8\sigma^2\log(\frac{2}{\delta})\log(2+t)}{t}}\Bigg) \leq \mathbb{P}\Bigg( |r(a) - \hat{r}_t(a)| > \sqrt{\frac{2\sigma^2\log(\frac{2}{\delta})}{N_t(a)}}\Bigg) \leq \delta \nonumber
\end{flalign}
by choosing a $\delta$ satisfying $\log(\frac{2}{\delta}) = \frac{1}{(\log(2+t))^3}$, we have
\begin{flalign}
\mathbb{P}\Bigg( |r(a) - \hat{r}_t(a)| > \sqrt{\frac{2\sigma^2\log(\frac{2}{\delta})}{N_t(a)}}\Bigg) \leq 2 \exp\Bigg( -\frac{1}{(\log(2+t))^3} \Bigg) \nonumber.
\end{flalign}
Therefore, for $t \geq 2$
\begin{flalign}
&\mathbb{P}\Bigg( \parallel r - \hat{r}_t \parallel_{\infty} > \frac{2\sigma}{ \log(2 + t)} \Bigg) \leq \mathbb{P}\Bigg( \parallel r - \hat{r}_t \parallel_{\infty} > \frac{2\sigma}{ \log(2 + t)} \Bigg| E_{\epsilon} \Bigg) + \mathbb{P}(E^C_{\epsilon}) \nonumber\\
&\leq \sum_{k} \Bigg( \mathbb{P}\Bigg( |r(a) - \hat{r}_t(a)| > \frac{2\sigma}{ \log(2 + t)}\Bigg) + \mathbb{P}(E^C_{\epsilon}) \leq 2|\mathcal{A}| \exp\Bigg( -\frac{1}{(\log(2+t))^3} \Bigg) \Bigg) \nonumber\\
& + 2|\mathcal{A}| \exp\Bigg( -\frac{t}{(\log(2+t))^3}\Bigg) = 4|\mathcal{A}| \exp\Bigg( -\frac{t}{(\log(2+t))^3}\Bigg) \nonumber.
\end{flalign}
\end{proof}

\begin{lemma}
Given two policies $\pi^{(1)} = \nabla \Omega^{*}(r^{(1)})$ and $\pi^{(2)} = \nabla \Omega^{*}(r^{(2)}), \exists L$, such that 
\begin{flalign}
\parallel \pi^{(1)} - \pi^{(2)} \parallel_{p} \leq L \parallel r^{(1)} - r^{(2)} \parallel_{p} \nonumber.
\end{flalign}
\end{lemma}
\begin{proof}
This comes directly from the fact that $\pi = \nabla \Omega^{*}(r)$ is Lipschitz continuous with $\ell^p$-norm. Note that $p$ has different values according to the choice of regularizer. Refer to~\cite{niculae2017regularized} for a discussion of each norm using maximum entropy and Tsallis entropy regularizer. Relative entropy shares the same properties with maximum Entropy.
\end{proof}
\begin{lemma}\label{lm_sum_pi}
Consider the E3W policy applied to a tree. At any node $s$ of the tree with depth $d$, Let us define $N_t^{*}(s, a) = \pi^{*}(a|s) . t$, and $\hat{N_t}(s, a) = \sum_{s=1}^{t}\pi_s(a|s)$, where $\pi_k(a|s)$  is the policy at time step $k$. There exists some $C$ and $\hat{C}$ such that
\begin{flalign}
\mathbb{P}\big( |\hat{N_t}(s, a) - N_t^{*}(s, a)| > \frac{Ct}{\log t}\big) \leq \hat{C} |\mathcal{A}| t \exp\{-\frac{t}{(\log t)^3}\} \nonumber.
\end{flalign}
\end{lemma}
\begin{proof}
We denote the following event,
\begin{flalign}
E_{r_k} = \{ \parallel r(s',\cdot) - \hat{r}_k(s',\cdot) \parallel_{\infty} < \frac{2 \sigma }{\log (2 + k)}\} \nonumber.
\end{flalign}
Thus, conditioned on the event $\bigcap_{i=1}^{t} E_{r_t}$ and for $t \geq 4$, we bound $|\hat{N_t}(s,a) - N^{*}_t(s,a)|$ as
\begin{flalign}
|\hat{N_t}(s, a) - N^{*}_t(s, a)| &\leq \sum_{k=1}^{t} |\hat{\pi}_k(a|s) - \pi^{*}(a|s)| + \sum_{k=1}^{t}\lambda_k \nonumber\\
&\leq \sum_{k=1}^{t} \parallel \hat{\pi}_k(\cdot|s) - \pi^{*}(\cdot|s) \parallel_{\infty} + \sum_{k=1}^{t}\lambda_k \nonumber\\
&\leq \sum_{k=1}^{t} \parallel \hat{\pi}_k(\cdot|s) - \pi^{*}(\cdot|s) \parallel_{p} + \sum_{k=1}^{t}\lambda_k \nonumber\\
&\leq L \sum_{k=1}^{t} \parallel \hat{Q}_k(s',\cdot) - Q(s',\cdot) \parallel_{p} + \sum_{k=1}^{t}\lambda_k (\text{Lemma 2}) \nonumber\\
&\leq L|\mathcal{A}|^{\frac{1}{p}} \sum_{k=1}^{t} \parallel \hat{Q}_k(s',\cdot) - Q(s',\cdot) \parallel_{\infty} + \sum_{k=1}^{t}\lambda_k (\text{ Property of $p$-norm}) \nonumber\\
&\leq L|\mathcal{A}|^{\frac{1}{p}} \gamma^{d}\sum_{k=1}^{t} \parallel \hat{r}_k(s'',\cdot) - r(s'',\cdot) \parallel_{\infty} + \sum_{k=1}^{t}\lambda_k (\text{Contraction~\ref{S:leg-fen}})\nonumber\\
&\leq L|\mathcal{A}|^{\frac{1}{p}} \gamma^{d} \sum_{k=1}^{t} \frac{2\sigma}{\log(2+k)} + \sum_{k=1}^{t}\lambda_k \nonumber\\
&\leq L|\mathcal{A}|^{\frac{1}{p}} \gamma^{d}\int_{k=0}^{t} \frac{2\sigma}{\log(2+k)} dk + \int_{k=0}^{t} \frac{|\mathcal{A}|}{\log(1 + k)} dk \nonumber\\
&\leq \frac{Ct}{\log t} \nonumber.
\end{flalign}
for some constant $C$ depending on $|\mathcal{A}|, p, d, \sigma, L$, and $\gamma$  . Finally,
\begin{flalign}
\mathbb{P}(|\hat{N_t}(s,a) - N^{*}_t(s,a)| \geq \frac{Ct}{\log t}) &\leq \sum_{i=1}^{t} \mathbb{P}(E^c_{r_t}) = \sum_{i=1}^{t} 4|\mathcal{A}| \exp (-\frac{t}{(\log(2 + t))^3}) \nonumber\\
&\leq 4|\mathcal{A}| t\exp (-\frac{t}{(\log(2 + t))^3}) \nonumber\\
&= O (t\exp (-\frac{t}{(\log(t))^3})) \nonumber.
\end{flalign}
\end{proof}

\begin{lemma}
Consider the E3W policy applied to a tree. At any node $s$ of the tree, Let us define $N_t^{*}(s, a) = \pi^{*}(a|s) . t$, and $N_t(s, a)$ as the number of times action $a$ have been chosen until time step $t$. There exists some $C$ and $\hat{C}$ such that
\begin{flalign}
\mathbb{P}\big( |N_t(s, a) - N_t^{*}(s, a)| > \frac{Ct}{\log t}\big) \leq \hat{C} t \exp\{-\frac{t}{(\log t)^3}\} \nonumber.
\end{flalign}
\end{lemma}

\begin{proof}
Based on the result from Lemma \ref{lm_sum_pi}, we have
\begin{align}
& \mathbb{P}\big( |N_t(s, a) - N_t^{*}(s, a)| > (1 + C)\frac{t}{\log t}\big) \leq C t \exp\{-\frac{t}{(\log t)^3}\} \nonumber\\
& \leq \mathbb{P}\big( |\hat{N}_t(s, a) - N_t^{*}(s, a)| > \frac{Ct}{\log t}\big) + \mathbb{P}\big( |N_t(s, a) - \hat{N}_t(s, a)| > \frac{t}{\log t}\big) \nonumber\\
& \leq 4|\mathcal{A}|t \exp\{-\frac{t}{(\log(2 + t))^3}\} + 2|\mathcal{A}| \exp\{-\frac{t}{(\log(2 + t))^2}\} (\text{Lemma } \ref{lm_sum_pi} \text{ and } (\ref{lb_subgaussion})) \nonumber\\
& \leq O(t\exp(-\frac{t}{(\log t)^3})) \nonumber.
\end{align}
\end{proof}

\begin{manualtheorem}{1}
At the root node $s$ of the tree, defining $N(s)$ as the number of visitations and $V_{\Omega^*}(s)$ as the estimated value at node s, for $\epsilon > 0$, we have
\begin{flalign}
\mathbb{P}(| V_{\Omega}(s) - V^{*}_{\Omega}(s) | > \epsilon) \leq C \exp\{-\frac{N(s) \epsilon }{\hat{C}(\log(2 + N(s)))^2}\} \nonumber.
\end{flalign}
\end{manualtheorem}
\begin{proof}

We prove this concentration inequality by induction. When the depth of the tree is $D = 1$, from Proposition \ref{lb_prop1}, we get
\begin{flalign}
| V_{\Omega}(s) - V^{*}_{\Omega}(s) | = \parallel \Omega^{*}(Q_{\Omega}(s,\cdot)) - \Omega^{*}(Q^{*}_{\Omega}(s,\cdot)) \parallel_{\infty} \leq \gamma \parallel \hat{r} - r^{*} \parallel_{\infty} (\text{Contraction}) \nonumber
\end{flalign}
where $\hat{r}$ is the average rewards and $r^*$ is the mean reward. So that
\begin{flalign}
\mathbb{P}(| V_{\Omega}(s) - V^{*}_{\Omega}(s) | > \epsilon) \leq  \mathbb{P}(\gamma\parallel \hat{r} - r^{*} \parallel_{\infty} > \epsilon) \nonumber.
\end{flalign}
From Lemma 1, with $\epsilon = \frac{2\sigma \gamma}{\log(2 + N(s))}$, we have
\begin{flalign}
\mathbb{P}(| V_{\Omega}(s) - V^{*}_{\Omega}(s) | > \epsilon) &\leq  \mathbb{P}(\gamma\parallel \hat{r} - r^{*} \parallel_{\infty} > \epsilon) \leq 4|\mathcal{A}| \exp\{-\frac{N(s) \epsilon}{2\sigma \gamma(\log(2 + N(s)))^2}\} \nonumber\\
&= C \exp\{-\frac{N(s) \epsilon}{\hat{C}(\log(2 + N(s)))^2}\} \nonumber.
\end{flalign}
Let assume we have the concentration bound at the depth $D-1$, 
Let us define $V_{\Omega}(s_a) = Q_{\Omega}(s,a)$, where $s_a$ is the state reached taking action $a$ from state $s$.
then at depth $D - 1$
\begin{flalign}
\mathbb{P}(| V_{\Omega}(s_a) - V^{*}_{\Omega}(s_a) | > \epsilon) \leq C \exp\{-\frac{N(s_a) \epsilon }{\hat{C}(\log(2 + N(s_a)))^2}\} \label{lb_d_1}.
\end{flalign}
Now at the depth $D$, because of the Contraction Property, we have
\begin{flalign}
| V_{\Omega}(s) - V^{*}_{\Omega}(s) | &\leq \gamma \parallel Q_{\Omega}(s,\cdot) - Q^{*}_{\Omega}(s,\cdot) \parallel_{\infty} \nonumber\\
&= \gamma | Q_{\Omega}(s,a) - Q^{*}_{\Omega}(s,a) | \nonumber.
\end{flalign} 
So that
\begin{flalign}
\mathbb{P}(| V_{\Omega}(s) - V^{*}_{\Omega}(s) | > \epsilon) &\leq  \mathbb{P}(\gamma \parallel Q_{\Omega}(s,a) - Q^{*}_{\Omega}(s,a) \parallel > \epsilon) \nonumber\\
&\leq  C_a \exp\{-\frac{N(s_a) \epsilon}{\hat{C_a}(\log(2 + N(s_a)))^2}\} \nonumber\\
&\leq  C_a \exp\{-\frac{N(s_a) \epsilon}{\hat{C_a}(\log(2 + N(s)))^2}\} \nonumber.
\end{flalign}
From (\ref{lb_d_1}), we can have $\lim_{t \rightarrow \infty} N(s_a) = \infty$ because if $\exists L, N(s_a) < L$, we can find $\epsilon > 0$ for which (\ref{lb_d_1}) is not satisfied.
From Lemma 4, when $N(s)$ is large enough, we have $N(s_a) \rightarrow \pi^*(a|s) N(s)$ (for example $N(s_a) > \frac{1}{2}\pi^*(a|s) N(s)$), that means we can find $C$ and $\hat{C}$ that satisfy
\begin{flalign}
\mathbb{P}(| V_{\Omega}(s) - V^{*}_{\Omega}(s) | > \epsilon) \leq C \exp\{-\frac{N(s) \epsilon }{\hat{C}(\log(2 + N(s)))^2}\} \nonumber.
\end{flalign}
\end{proof}

\begin{lemma}
At any node $s$ of the tree, $N(s)$ is the number of visitations.
We define the event
\begin{flalign}
E_s = \{ \forall a \in \mathcal{A}, |N(s,a) - N^{*}(s,a)| < \frac{N^{*}(s,a)}{2} \} \nonumber \text{ where }N^{*}(s,a) = \pi^*(a|s)N(s),
\end{flalign}
where $\epsilon > 0$ and $V_{\Omega^*}(s)$ is the estimated value at node s. We have
\begin{flalign}
\mathbb{P}(| V_{\Omega}(s) - V^{*}_{\Omega}(s) | > \epsilon|E_s) \leq C \exp\{-\frac{N(s) \epsilon }{\hat{C}(\log(2 + N(s)))^2}\} \nonumber.
\end{flalign}
\end{lemma}
\begin{proof}
The proof is the same as in Theorem 2. We prove the concentration inequality by induction. When the depth of the tree is $D = 1$, from Proposition \ref{lb_prop1}, we get
\begin{flalign}
| V_{\Omega}(s) - V^{*}_{\Omega}(s) | = \parallel \Omega^{*}(Q_{\Omega}(s,\cdot)) - \Omega^{*}(Q^{*}_{\Omega}(s,\cdot)) \parallel \leq \gamma \parallel \hat{r} - r^{*} \parallel_{\infty} \nonumber (\text{Contraction Property})
\end{flalign}
where $\hat{r}$ is the average rewards and $r^*$ is the mean rewards. So that
\begin{flalign}
\mathbb{P}(| V_{\Omega}(s) - V^{*}_{\Omega}(s) | > \epsilon) \leq  \mathbb{P}(\gamma\parallel \hat{r} - r^{*} \parallel_{\infty} > \epsilon) \nonumber.
\end{flalign}
From Lemma 1, with $\epsilon = \frac{2\sigma \gamma}{\log(2 + N(s))}$ and given $E_s$, we have
\begin{flalign}
\mathbb{P}(| V_{\Omega}(s) - V^{*}_{\Omega}(s) | > \epsilon) &\leq  \mathbb{P}(\gamma\parallel \hat{r} - r^{*} \parallel_{\infty} > \epsilon) \leq 4|\mathcal{A}| \exp\{-\frac{N(s) \epsilon}{2\sigma \gamma(\log(2 + N(s)))^2}\} \nonumber \\
&= C \exp\{-\frac{N(s) \epsilon}{\hat{C}(\log(2 + N(s)))^2}\} \nonumber.
\end{flalign}
Let assume we have the concentration bound at the depth $D-1$, 
Let us define $V_{\Omega}(s_a) = Q_{\Omega}(s,a)$, where $s_a$ is the state reached taking action $a$ from state $s$, then at depth $D - 1$
\begin{flalign}
\mathbb{P}(| V_{\Omega}(s_a) - V^{*}_{\Omega}(s_a) | > \epsilon) \leq C \exp\{-\frac{N(s_a) \epsilon }{\hat{C}(\log(2 + N(s_a)))^2}\} \nonumber.
\end{flalign}
Now at depth $D$, because of the Contraction Property and given $E_s$, we have
\begin{flalign}
| V_{\Omega}(s) - V^{*}_{\Omega}(s) | &\leq \gamma \parallel Q_{\Omega}(s,\cdot) - Q^{*}_{\Omega}(s,\cdot) \parallel_{\infty} \nonumber\\
&= \gamma | Q_{\Omega}(s,a) - Q^{*}_{\Omega}(s,a) | \nonumber (\exists a, \text{ satisfied}).
\end{flalign}
So that
\begin{flalign}
\mathbb{P}(| V_{\Omega}(s) - V^{*}_{\Omega}(s) | > \epsilon) &\leq  \mathbb{P}(\gamma \parallel Q_{\Omega}(s,a) - Q^{*}_{\Omega}(s,a) \parallel > \epsilon) \nonumber \\
&\leq  C_a \exp\{-\frac{N(s_a) \epsilon}{\hat{C_a}(\log(2 + N(s_a)))^2}\} \nonumber\\
&\leq  C_a \exp\{-\frac{N(s_a) \epsilon}{\hat{C_a}(\log(2 + N(s)))^2}\} \nonumber\\
&\leq C \exp\{-\frac{N(s) \epsilon }{\hat{C}(\log(2 + N(s)))^2}\} \nonumber (\text{because of $E_s$})
\end{flalign}.
\end{proof}

\begin{manualtheorem}{2}
Let $a_t$ be the action returned by algorithm E3W at iteration $t$. Then for $t$ large enough, with some constants $C, \hat{C}$,
\begin{flalign}
&\mathbb{P}(a_t \neq a^{*}) \leq C t \exp\{-\frac{t}{\hat{C} \sigma (\log( t))^3}\} \nonumber.
\end{flalign}
\end{manualtheorem}
\begin{proof}
Let us define event $E_s$ as in Lemma 5.
Let $a^{*}$ be the action with largest value estimate at the root node state $s$. The probability that E3W selects a sub-optimal arm at $s$ is
\begin{flalign}
&\mathbb{P}(a_t \neq a^{*}) \leq \sum_a \mathbb{P}( V_{\Omega}(s_a)) > V_{\Omega}(s_{a^*})|E_s) + \mathbb{P}(E_s^c) \nonumber\\
&= \sum_a \mathbb{P}( (V_{\Omega}(s_a) - V^{*}_{\Omega}(s_a)) - (V_{\Omega}(s_{a^*}) - V^{*}_{\Omega}(s_{a^*})) \geq V^{*}_{\Omega}(s_{a^*}) - V^{*}_{\Omega}(s_a)|E_s) + \mathbb{P}(E_s^c) \nonumber.
\end{flalign}
Let us define $\Delta = V^{*}_{\Omega}(s_{a^*}) - V^{*}_{\Omega}(s_a)$, therefore for $\Delta > 0$, we have
\begin{flalign}
&\mathbb{P}(a_t \neq a^{*}) \leq \sum_a \mathbb{P}( (V_{\Omega}(s_a) - V^{*}_{\Omega}(s_a)) - (V_{\Omega}(s_{a^*}) - V^{*}_{\Omega}(s_{a^*})) \geq \Delta|E_s) + + \mathbb{P}(E_s^c) \nonumber\\
&\leq \sum_a \mathbb{P}( |V_{\Omega}(s_a) - V^{*}_{\Omega}(s_a)| \geq \alpha \Delta | E_s) + \mathbb{P} (|V_{\Omega}(s_{a^*}) - V^{*}_{\Omega}(s_{a^*})| \geq \beta \Delta|E_s) + \mathbb{P}(E_s^c) \nonumber\\
&\leq \sum_a C_a \exp\{-\frac{N(s) (\alpha \Delta) }{\hat{C_a}(\log(2 + N(s)))^2}\} + C_{a^{*}} \exp\{-\frac{N(s) (\beta \Delta)}{\hat{C}_{a^*}(\log(2 + N(s)))^2}\} + \mathbb{P}(E_s^c) \nonumber,
\end{flalign}

where $\alpha + \beta = 1$, $\alpha > 0$, $\beta > 0$, and $N(s)$ is the number of visitations the root node $s$.
Let us define $\frac{1}{\hat{C}} = \min \{ \frac{(\alpha\Delta)}{C_a}, \frac{(\beta\Delta)}{C_{a^*}}\}$, and $C = \frac{1}{|\mathcal{A}|}\max \{C_a, C_{a^*}\}$
we have
\begin{flalign}
&\mathbb{P}(a \neq a^{*}) \leq C \exp\{-\frac{t}{\hat{C} \sigma (\log(2 + t))^2}\} + \mathbb{P}(E_s^c) \nonumber.
\end{flalign}
From Lemma 4, $\exists C^{'}, \hat{C^{'}}$ for which
\begin{flalign}
\mathbb{P}(E_s^c) \leq C^{'} t\exp\{-\frac{t}{\hat{C^{'}} (\log(t))^3}\} \nonumber,
\end{flalign}
so that
\begin{flalign}
&\mathbb{P}(a \neq a^{*}) \leq O (t\exp\{-\frac{t}{(\log(t))^3}\}) \nonumber.
\end{flalign}
\end{proof}

\begin{manualtheorem}{3}
Consider an E3W policy applied to the tree. Let define $\mathcal{D}_{\Omega^*}(x,y) = \Omega^*(x) - \Omega^*(y) - \nabla \Omega^* (y) (x - y)$ as the Bregman divergence between $x$ and $y$, The expected pseudo regret $R_n$ satisfies
\begin{flalign}
\mathbb{E}[R_n] \leq - \tau \Omega(\hat{\pi}) + \sum_{t=1}^n \mathcal{D}_{\Omega^*} (\hat{V_t}(\cdot) + V(\cdot), \hat{V_t}(\cdot)) + \mathcal{O} (\frac{n}{\log n})\nonumber.
\end{flalign}
\end{manualtheorem}

\begin{proof}
Without loss of generality, we can assume that $V_i \in [-1, 0], \forall i \in [1, |A|]$.
as the definition of regret, we have
\begin{align}
\mathbb{E}[R_n] = nV^{*} - \sum_{t=1}^n  \left\langle \hat{\pi}_t(\cdot), V(\cdot)\right\rangle \leq \hat{V}_1(0) - \sum_{t=1}^n  \left\langle\hat{\pi}_t(\cdot), V(\cdot)\right\rangle \leq -\tau \Omega(\hat{\pi}) - \sum_{t=1}^n  \left\langle\hat{\pi}_t(\cdot), V(\cdot)\right\rangle. \nonumber
\end{align}
By the definition of the tree policy, we can obtain
\begin{flalign}
   - \sum_{t=1}^n \left\langle\hat{\pi}_t(\cdot), V(\cdot)\right\rangle &= - \sum_{t=1}^n \left\langle (1-\lambda_t)\nabla \Omega^*(\hat{V_t}(\cdot)), V(\cdot)\right\rangle - \sum_{t=1}^n \left\langle \frac{\lambda_t (\cdot)}{|A|}, V(\cdot)\right\rangle \nonumber \\
   &= -\sum_{t=1}^n \left\langle (1-\lambda_t)\nabla \Omega^*(\hat{V_t}(\cdot)), V(\cdot)\right\rangle - \sum_{t=1}^n \left\langle \frac{\lambda_t (\cdot)}{|A|}, V(\cdot)\right\rangle \nonumber \\
   &\leq -\sum_{t=1}^n \left\langle \nabla \Omega^*(\hat{V_t}(\cdot)), V(\cdot)\right\rangle - \sum_{t=1}^n \left\langle \frac{\lambda_t (\cdot)}{|A|}, V(\cdot)\right\rangle. \nonumber
\end{flalign}
with
\begin{flalign}
   - \sum_{t=1}^n \left\langle \nabla \Omega^*(\hat{V_t}(\cdot)), V(\cdot)\right\rangle &= \sum_{t = 1}^n \Omega^* (\hat{V_t}(\cdot) + V(\cdot)) - \sum_{t = 1}^n \Omega^* (\hat{V_t}(\cdot)) - \sum_{t=1}^n \left\langle \nabla \Omega^*(\hat{V_t}(\cdot)), V(\cdot)\right\rangle \nonumber\\
   &- ( \sum_{t = 1}^n \Omega^* (\hat{V_t}(\cdot) + V(\cdot)) - \sum_{t = 1}^n \Omega^* (\hat{V_t}(\cdot)) ) \nonumber \\
   &= \sum_{t=1}^n \mathcal{D}_{\Omega^*} (\hat{V_t}(\cdot) + V(\cdot), \hat{V_t}(\cdot)) - ( \sum_{t = 1}^n \Omega^* (\hat{V_t}(\cdot) + V(\cdot)) - \sum_{t = 1}^n \Omega^* (\hat{V_t}(\cdot)) ) \nonumber \\
   &\leq \sum_{t=1}^n \mathcal{D}_{\Omega^*} (\hat{V_t}(\cdot) + V(\cdot), \hat{V_t}(\cdot)) + n \parallel V(\cdot) \parallel_{\infty} (\text{Contraction property, Proposition 1})\nonumber\\
   &\leq \sum_{t=1}^n \mathcal{D}_{\Omega^*} (\hat{V_t}(\cdot) + V(\cdot), \hat{V_t}(\cdot)). (\text{ because }V_i \leq 0) \nonumber
\end{flalign}
And
\begin{flalign}
    - \sum_{t=1}^n \left\langle \frac{\lambda_t (\cdot)}{|A|}, V(\cdot)\right\rangle \leq \mathcal{O} (\frac{n}{\log n}), (\text{Because} \sum_{k = 1}^n \frac{1}{\log (k + 1)} \rightarrow \mathcal{O} (\frac{n}{\log n})) \nonumber
\end{flalign}
So that
\begin{flalign}
   \mathbb{E}[R_n] &\leq - \tau \Omega(\hat{\pi}) + \sum_{t=1}^n \mathcal{D}_{\Omega^*} (\hat{V_t}(\cdot) + V(\cdot), \hat{V_t}(\cdot)) + \mathcal{O} (\frac{n}{\log n}). \nonumber
\end{flalign}
We consider the generalized Tsallis Entropy $\Omega(\pi) = \mathcal{S}_{\alpha} (\pi) = \frac{1}{1 - \alpha} (1 - \sum_i \pi^{\alpha} (a_i|s))$.\\
According to \citep{abernethy2015fighting}, when $\alpha \in (0,1)$
\begin{flalign}
   \mathcal{D}_{\Omega^*} (\hat{V_t}(\cdot) + V(\cdot), \hat{V_t}(\cdot)) \leq (\tau \alpha) ^ {-1} |\mathcal{A}|^{\alpha} \nonumber \\
   -\Omega (\hat{\pi}_n) \leq \frac{1}{1-\alpha}(|\mathcal{A}|^{1-\alpha} - 1). \nonumber
\end{flalign}
Then, for the generalized Tsallis Entropy, when $\alpha \in (0,1)$, the regret is 
\begin{flalign}
   \mathbb{E}[R_n] &\leq \frac{\tau}{1-\alpha} (|\mathcal{A}|^{1-\alpha} - 1) + n(\tau \alpha) ^ {-1} |\mathcal{A}|^{\alpha} + \mathcal{O} (\frac{n}{\log n}), \nonumber
\end{flalign}
when $\alpha = 2$, which is the Tsallis entropy case we consider, according to \cite{zimmert2019optimal}, By Taylor's theorem $\exists z \in \text{conv}(\hat{V}_t, \hat{V}_t + V)$,  we have
\begin{flalign}
    \mathcal{D}_{\Omega^*} (\hat{V_t}(\cdot) + V(\cdot), \hat{V_t}(\cdot)) \leq \frac{1}{2} \left  \langle V(\cdot), \nabla^2 \Omega^* (z) V(\cdot) \right\rangle  \leq \frac{|\mathcal{K}|}{2}.\nonumber
\end{flalign}
So that when $\alpha = 2$, we have
\begin{flalign}
   \mathbb{E}[R_n] &\leq \tau (\frac{|\mathcal{A}| - 1}{|\mathcal{A}|}) + \frac{n|\mathcal{K}|}{2} + \mathcal{O} (\frac{n}{\log n}). \nonumber
\end{flalign}
when $\alpha = 1$, which is the maximum entropy case in our paper, we derive.
\begin{flalign}
   \mathbb{E}[R_n] &\leq \tau (\log |\mathcal{A}|) + \frac{n|\mathcal{A}|}{\tau} + \mathcal{O} (\frac{n}{\log n}) \nonumber
\end{flalign}
Finally, when the convex regularizer is relative entropy, One can simply write $KL(\pi_t || \pi_{t-1}) = -H(\pi_t) - \mathbb{E}_{\pi_t} \log \pi_{t-1}$, let $m = \min_{a} \pi_{t-1}(a|s)$, we have
\begin{flalign}
   \mathbb{E}[R_n] &\leq \tau (\log |\mathcal{A}| - \frac{1}{m}) + \frac{n|\mathcal{A}|}{\tau} + \mathcal{O} (\frac{n}{\log n}). \nonumber
\end{flalign}
\end{proof}


Before derive the next theorem, we state the Theorem 2 in~\cite{geist2019theory}
\begin{itemize}
\item Boundedness: for two constants $L_{\Omega}$ and $U_{\Omega}$ such that for all $\pi \in \Pi$, we have $L_{\Omega} \leq \Omega(\pi) \leq U_{\Omega}$, then
    \begin{flalign}
        V^{*}(s) - \frac{\tau (U_{\Omega} - L_{\Omega})}{1 - \gamma} \leq V^{*}_{\Omega}(s) \leq V^{*}(s).
    \end{flalign}
\end{itemize}
Where $\tau$ is the temperature and $\gamma$ is the discount constant.
\begin{manualtheorem}{4}
For any $\delta > 0$, with probability at least $1 - \delta$, the $\varepsilon_{\Omega}$ satisfies
\begin{flalign}
&-\sqrt{\frac{\Hat{C}\sigma^2\log\frac{C}{\delta}}{2N(s)}} - \frac{\tau(U_{\Omega} - L_{\Omega})}{1 - \gamma} \leq \varepsilon_{\Omega}  \leq \sqrt{\frac{\Hat{C}\sigma^2\log\frac{C}{\delta}}{2N(s)}}  \nonumber.
\end{flalign}
\end{manualtheorem}
\begin{proof}
From Theorem 2, let us define $\delta = C \exp\{-\frac{2N(s) \epsilon^2}{\hat{C}\sigma^2}\}$, so that $\epsilon = \sqrt{\frac{\Hat{C}\sigma^2\log\frac{C}{\delta}}{2N(s)}}$ then for any $\delta > 0$, we have
\begin{flalign}
\mathbb{P}(|V_{\Omega}(s) - V^{*}_{\Omega}(s)| \leq \sqrt{\frac{\Hat{C}\sigma^2\log\frac{C}{\delta}}{2N(s)}}) \geq 1 - \delta \nonumber.
\end{flalign}
Then, for any $\delta > 0$, with probability at least $1 - \delta$, we have 
\begin{flalign}
&|V_{\Omega}(s) - V^{*}_{\Omega}(s)| \leq \sqrt{\frac{\Hat{C}\sigma^2\log\frac{C}{\delta}}{2N(s)}} \nonumber\\
&-\sqrt{\frac{\Hat{C}\sigma^2\log\frac{C}{\delta}}{2N(s)}} \leq V_{\Omega}(s) - V^{*}_{\Omega}(s) \leq \sqrt{\frac{\Hat{C}\sigma^2\log\frac{C}{\delta}}{2N(s)}} \nonumber\\
&-\sqrt{\frac{\Hat{C}\sigma^2\log\frac{C}{\delta}}{2N(s)}} +  V^{*}_{\Omega}(s)\leq V_{\Omega}(s)  \leq \sqrt{\frac{\Hat{C}\sigma^2\log\frac{C}{\delta}}{2N(s)}}  
+ V^{*}_{\Omega}(s) \nonumber.
\end{flalign}

From Proposition 1, we have
\begin{flalign}
&-\sqrt{\frac{\Hat{C}\sigma^2\log\frac{C}{\delta}}{2N(s)}} + V^{*}(s) - \frac{\tau(U_{\Omega} - L_{\Omega})}{1 - \gamma} \leq V_{\Omega}(s)  \leq \sqrt{\frac{\Hat{C}\sigma^2\log\frac{C}{\delta}}{2N(s)}} 
+ V^{*}(s) \nonumber.
\end{flalign}
\end{proof}